\documentclass[letterpaper]{article}
\usepackage{algorithm}
\usepackage[noend]{algorithmic}
\usepackage{amsmath}
\usepackage{amssymb}
\usepackage{amsthm}
\usepackage{bbm}
\usepackage{bm}
\usepackage{color}
\usepackage{dsfont}
\usepackage{enumerate}
\usepackage{graphicx}
\usepackage{hyperref}
\usepackage[accepted]{icml2016}
\usepackage{natbib}
\usepackage{subfigure}
\usepackage{times}
\usepackage[usenames,dvipsnames]{xcolor}

\usepackage[capitalize]{cleveref}

\newcommand{\commentout}[1]{}
\newcommand{\junk}[1]{}

\newtheorem{theorem}{Theorem}

\newtheorem{lemma}{Lemma}
\newtheorem{proposition}{Proposition}

\newcommand{\cH}{\mathcal{H}}
\newcommand{\eps}{\varepsilon}

\newcommand{\condEE}[2]{\mathbb{E} \left[#1 \,\middle|\, #2\right]}

\newcommand{\EE}[1]{\mathbb{E} \left[#1\right]}

\newcommand{\kl}[2]{D_\mathrm{KL}(#1 \,\|\, #2)}

\newcommand{\rnd}[1]{\mathbf{#1}}
\newcommand{\set}[1]{\left\{#1\right\}}

\DeclareMathOperator*{\argmax}{arg\,max\,}

\mathchardef\mhyphen="2D

\newcommand{\cascadeklucb}{{\tt CascadeKL\mhyphen UCB}}

\newcommand{\dcmklucb}{{\tt dcmKL\mhyphen UCB}}
\newcommand{\expthree}{{\tt Exp3}}
\newcommand{\firstclick}{{\tt First\mhyphen Click}}
\newcommand{\klucb}{{\tt KL\mhyphen UCB}}
\newcommand{\lastclick}{{\tt Last\mhyphen Click}}
\newcommand{\rankedexpthree}{{\tt RankedExp3}}
\newcommand{\rankedklucb}{{\tt RankedKL\mhyphen UCB}}

\begin{document}

\icmltitlerunning{DCM Bandits: Learning to Rank with Multiple Clicks}

\twocolumn[
\icmltitle{DCM Bandits: Learning to Rank with Multiple Clicks}
\icmlauthor{Sumeet Katariya}{katariya@wisc.edu}
\icmladdress{Department of Electrical and Computer Engineering, University of Wisconsin-Madison}
\icmlauthor{Branislav Kveton}{kveton@adobe.com}
\icmladdress{Adobe Research, San Jose, CA}
\icmlauthor{Csaba Szepesv\'ari}{szepesva@cs.ualberta.ca}
\icmladdress{Department of Computing Science, University of Alberta}
\icmlauthor{Zheng Wen}{zwen@adobe.com}
\icmladdress{Adobe Research, San Jose, CA}
\vskip 0.3in]

\begin{abstract}
A search engine recommends to the user a list of web pages. The user examines this list, from the first page to the last, and clicks on all attractive pages until the user is satisfied. This behavior of the user can be described by the \emph{dependent click model (DCM)}. We propose \emph{DCM bandits}, an online learning variant of the DCM where the goal is to maximize the probability of recommending satisfactory items, such as web pages. The main challenge of our learning problem is that we do not observe which attractive item is satisfactory. We propose a computationally-efficient learning algorithm for solving our problem, $\dcmklucb$; derive gap-dependent upper bounds on its regret under reasonable assumptions; and also prove a matching lower bound up to logarithmic factors. We evaluate our algorithm on synthetic and real-world problems, and show that it performs well even when our model is misspecified. This work presents the first practical and regret-optimal online algorithm for learning to rank with multiple clicks in a cascade-like click model.
\end{abstract}

%!TEX root = Paper.tex

\section{Introduction}
\label{sec:introduction}

Web pages in search engines are often ranked based on a model of user behavior, which is learned from click data \cite{radlinski05query,agichtein06improving,chuklin15click}. The cascade model \cite{craswell08experimental} is one of the most popular models of user behavior in web search. \citet{kveton15cascading} and \citet{combes15learning} recently proposed regret-optimal online learning algorithms for the cascade model. The main limitation of the cascade model is that it cannot model multiple clicks. Although the model was extended to multiple clicks \cite{chapelle09dynamic,guo09click,guo09efficient}, it is unclear if it is possible to design computationally and sample efficient online learning algorithms for these extensions.

In this work, we propose an online learning variant of the \emph{dependent click model (DCM)} \cite{guo09efficient}, which we call \emph{DCM bandits}. The DCM is a generalization of the cascade model where the user may click on multiple items. At time $t$, our learning agent recommends to the user a list of $K$ items. The user examines this list, from the first item to the last. If the examined item attracts the user, the user clicks on it. This is observed by the learning agent. After the user clicks on the item and investigates it, the user may leave or examine more items. If the user leaves, the DCM interprets this as that the user is satisfied and our agent receives a reward of one. If the user examines all items and does not leave on purpose, our agent receives a reward of zero. The goal of the agent is to maximize its total reward, or equivalently to minimize its cumulative regret with respect to the most satisfactory list of $K$ items. Our learning problem is challenging because the agent does not observe whether the user is satisfied. It only observes clicks. This differentiates our problem from cascading bandits \cite{kveton15cascading}, where the user can click on at most one item and this click is satisfactory.

We make four major contributions. First, we formulate an online learning variant of the DCM. Second, we propose a computationally-efficient learning algorithm for our problem under the assumption that the order of the termination probabilities in the DCM is known. Our algorithm is motivated by $\klucb$ \cite{garivier11klucb}, and therefore we call it $\dcmklucb$. Third, we prove two gap-dependent upper bounds on the regret of $\dcmklucb$ and a matching lower bound up to logarithmic factors. The key step in our analysis is a novel reduction to cascading bandits \cite{kveton15cascading}. Finally, we evaluate our algorithm on both synthetic and real-world problems, and compare it to several baselines. We observe that $\dcmklucb$ performs well even when our modeling assumptions are violated.

We denote random variables by boldface letters and write $[n]$ to denote $\set{1, \dots, n}$. For any sets $A$ and $B$, we denote by $A^B$ the set of all vectors whose entries are indexed by $B$ and take values from $A$.

%!TEX root = Paper.tex

\section{Background}
\label{sec:background}

\begin{figure}[t]
  \includegraphics[width=.9\columnwidth]{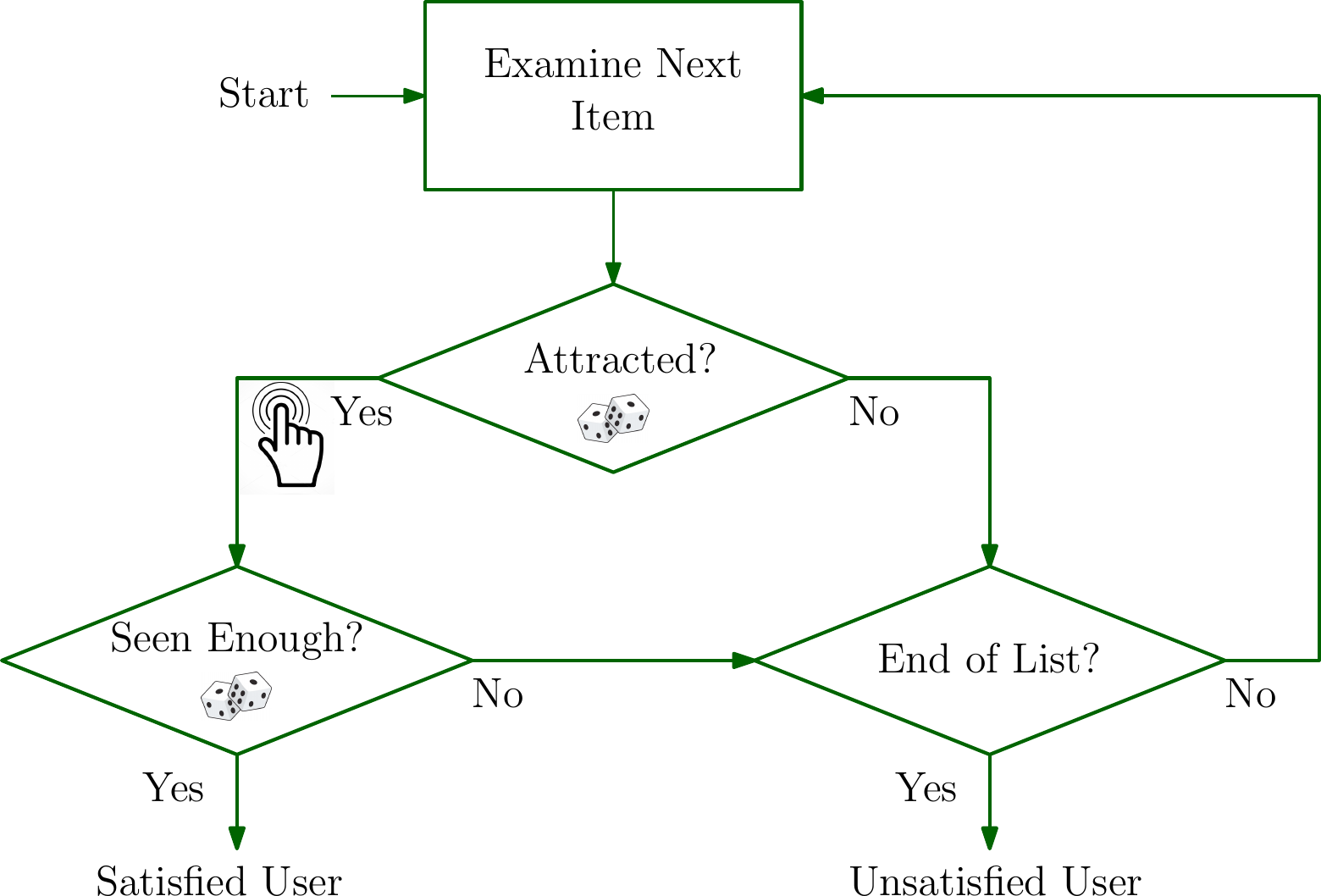}
  \caption{Interaction between the user and items in the DCM.}
  \label{fig:DCM}
\end{figure}

Web pages in search engines are often ranked based on a model of user behavior, which is learned from click data \cite{radlinski05query,agichtein06improving,chuklin15click}. We assume that the user scans a list of $K$ web pages $A = (a_1, \dots, a_K)$, which we call \emph{items}. These items belong to some \emph{ground set} $E = [L]$, such as the set of all possible web pages. Many models of user behavior in web search exist \cite{becker07modeling,richardson07predicting,craswell08experimental,chapelle09dynamic,guo09click,guo09efficient}. We focus on the dependent click model.

The \emph{dependent click model (DCM)} \cite{guo09efficient} is an extension of the cascade model \cite{craswell08experimental} to multiple clicks. The model assumes that the user scans a list of $K$ items $A = (a_1, \dots, a_K) \in \Pi_K(E)$ from the first item $a_1$ to the last $a_K$, where $\Pi_K(E) \subset E^K$ is the set of all \emph{$K$-permutations} of $E$. The DCM is parameterized by $L$ \emph{item-dependent attraction probabilities} $\bar{w} \in [0, 1]^E$ and $K$ \emph{position-dependent termination probabilities} $\bar{v} \in [0, 1]^K$. After the user \emph{examines} item $a_k$, the item \emph{attracts} the user with probability $\bar{w}(a_k)$. If the user is attracted by item $a_k$, the user clicks on the item and \emph{terminates} the search with probability $\bar{v}(k)$. If this happens, the user is \emph{satisfied} with item $a_k$ and does not examine any of the \emph{remaining} items. If item $a_k$ is not attractive or the user does not terminate, the user examines item $a_{k + 1}$. Our interaction model is visualized in Figure~\ref{fig:DCM}.

Before we proceed, we would like to stress the following. First, all probabilities in the DCM are independent of each other. Second, the probabilities $\bar{w}(a_k)$ and $\bar{v}(k)$ are \emph{conditioned} on the events that the user examines position $k$ and that the examined item is attractive, respectively. For simplicity of exposition, we drop ``conditional'' in this paper. Finally, $\bar{v}(k)$ is \emph{not} the probability that the user terminates at position $k$. This latter probability depends on the items and positions before position $k$.

It is easy to see that the probability that the user is satisfied with list $A = (a_1, \dots, a_K)$ is $1 - \prod_{k = 1}^{K} (1 - \bar{v}(k) \bar{w}({a_k}))$. This objective is maximized when the $k$-th most attractive item is placed at the $k$-th most terminating position.

%!TEX root = Paper.tex

\section{DCM Bandits}
\label{sec:DCM bandits}

We propose a learning variant of the dependent click model (\cref{sec:setting}) and a computationally-efficient algorithm for solving it (\cref{sec:algorithm}).

\subsection{Setting}
\label{sec:setting}

We refer to our learning problem as a \emph{DCM bandit}. Formally, we define it as a tuple $B = (E, P_\textsc{w}, P_\textsc{v}, K)$, where $E = [L]$ is a \emph{ground set} of $L$ items; $P_\textsc{w}$ and $P_\textsc{v}$ are probability distributions over $\set{0, 1}^E$ and $\set{0, 1}^K$, respectively; and $K \leq L$ is the number of recommended items.

The learning agent interacts with our problem as follows. Let $(\rnd{w}_t)_{t = 1}^n$ be $n$ i.i.d. \emph{attraction weights} drawn from distribution $P_\textsc{w}$, where $\rnd{w}_t \in \set{0, 1}^E$ and $\rnd{w}_t(e)$ indicates that item $e$ is attractive at time $t$; and let $(\rnd{v}_t)_{t = 1}^n$ be $n$ i.i.d. \emph{termination weights} drawn from $P_\textsc{v}$, where $\rnd{v}_t \in \set{0, 1}^K$ and $\rnd{v}_t(k)$ indicates that the user would terminate at position $k$ if the item at that position was examined and attractive. At time $t$, the learning agent recommends to the user a list of $K$ items $\rnd{A}_t = (\rnd{a}^t_1, \dots, \rnd{a}^t_K) \in \Pi_K(E)$. The user examines the items in the order in which they are presented and the agent receives observations $\rnd{c}_t \in \set{0, 1}^K$ that indicate the clicks of the user. Specifically, $\rnd{c}_t(k) = 1$ if and only if the user clicks on item $\rnd{a}^t_k$, the item at position $k$ at time $t$.

The learning agent also receives a \emph{binary reward} $\rnd{r}_t$, which is \emph{unobserved}. The reward is one if and only if the user is satisfied with at least one item in $\rnd{A}_t$. We say that item $e$ is \emph{satisfactory} at time $t$ when it is attractive, $\rnd{w}_t(e) = 1$, and its position leads to termination, $\rnd{v}_t(k) = 1$. The reward can be written as $\rnd{r}_t = f(\rnd{A}_t, \rnd{w}_t, \rnd{v}_t)$, where $f: \Pi_K(E) \times\allowbreak [0, 1]^E \times [0, 1]^K \to [0, 1]$ is a \emph{reward function}, which we define as
\begin{align*}
  f(A, w, v) = 1 - \prod_{k = 1}^K (1 - v(k) w(a_k))
\end{align*}
for any $A = (a_1, \dots, a_K) \in \Pi_K(E)$, $w \in [0, 1]^E$, and $v \in\allowbreak [0, 1]^K$. The above form is very useful in our analysis.

\citet{guo09efficient} assume that the attraction and termination weights in the DCM are drawn \emph{independently} of each other. We also adopt these assumptions. More specifically, we assume that for any $w \in \set{0, 1}^E$ and $v \in \set{0, 1}^K$,
\begin{align*}
  P_\textsc{w}(w)
  & = \prod\nolimits_{e \in E} \mathrm{Ber}(w(e); \bar{w}(e))\,, \\
  P_\textsc{v}(v)
  & = \prod\nolimits_{k \in [K]} \mathrm{Ber}(v(k); \bar{v}(k))\,,
\end{align*}
where $\mathrm{Ber}(\cdot; \theta)$ is a Bernoulli probability distribution with mean $\theta$. The above assumptions allow us to design a very efficient learning algorithm. In particular, they imply that the expected reward for list $A$, the probability that at least one item in $A$ is satisfactory, decomposes as
\begin{align*}
  \EE{f(A, \rnd{w}, \rnd{v})}
  & = 1 - \prod_{k = 1}^K (1 - \EE{\rnd{v}(k)} \EE{\rnd{w}(a_k)}) \\
  & = f(A, \bar{w}, \bar{v})
\end{align*}
and depends only on the attraction probabilities of items in $A$ and the termination probabilities $\bar{v}$. An analogous property proved useful in the design and analysis of algorithms for cascading bandits \cite{kveton15cascading}.

We evaluate the performance of a learning agent by its \emph{expected cumulative regret} $R(n) = \EE{\sum_{t = 1}^n R(\rnd{A}_t, \rnd{w}_t, \rnd{v}_t)}$, where $R(\rnd{A}_t, \rnd{w}_t, \rnd{v}_t) = f(A^\ast, \rnd{w}_t, \rnd{v}_t) - f(\rnd{A}_t, \rnd{w}_t, \rnd{v}_t)$ is the \emph{instantaneous regret} of the agent at time $t$ and
\begin{align*}
  \textstyle
  A^\ast = \argmax_{A \in \Pi_K(E)} f(A, \bar{w}, \bar{v})
\end{align*}
is the \emph{optimal list} of items, the list that maximizes the expected reward. Note that $A^\ast$ is the list of $K$ most attractive items, where the $k$-th most attractive item is placed at the $k$-th most terminating position. To simplify exposition, we assume that the optimal solution, as a set, is unique.

\subsection{Learning Without Accessing Rewards}
\label{sec:learning without accessing rewards}

Learning in DCM bandits is difficult because the observations $\rnd{c}_t$ are not sufficient to determine the reward $\rnd{r}_t$. We illustrate this problem on the following example. Suppose that the agent recommends $\rnd{A}_t = (1, 2, 3, 4)$ and observes $\rnd{c}_t = (0, 1, 1, 0)$. This feedback can be interpreted as follows. The first explanation is that item $1$ is not attractive, items $2$ and $3$ are, and the user does not terminate at position $3$. The second explanation is that item $1$ is not attractive, items $2$ and $3$ are, and the user terminates at position $3$. In the first case, the reward is zero. In the second case, the reward is one. Since the rewards are unobserved, DCM bandits are an instance of \emph{partial monitoring} (\cref{sec:related work}). However, general algorithms for partial monitoring are not suitable for DCM bandits because their number of actions is exponential in $K$. Therefore, we make an additional assumption that allows us to learn efficiently.

The key idea in our solution is based on the following insight. Without loss of generality, suppose that the termination probabilities satisfy $\bar{v}(1) \geq \ldots \geq\allowbreak \bar{v}(K)$. Then $A^\ast =\allowbreak \argmax_{A \in \Pi_K(E)} f(A, \bar{w}, \tilde{v})$ for any $\tilde{v} \in [0, 1]^K$ such that $\tilde{v}(1) \geq \ldots \geq \tilde{v}(K)$. Therefore, the \emph{termination probabilities} do not have to be learned if their \emph{order is known}, and we assume this in the rest of the paper. This assumption is much milder than knowing the probabilities. In \cref{sec:experiments}, we show that our algorithm performs well even when this order is misspecified.

Finally, we need one more insight. Let
\begin{align*}
  \rnd{C}^\mathrm{last}_t = \max \set{k \in [K]: \rnd{c}_t(k) = 1}
\end{align*}
denote the position of the \emph{last click}, where $\max \emptyset = + \infty$. Then $\rnd{w}_t(\rnd{a}^t_k) = \rnd{c}_t(k)$ for any $k \leq \min \set{\rnd{C}^\mathrm{last}_t, K}$. This means that the first $\min \set{\rnd{C}^\mathrm{last}_t, K}$ entries of $\rnd{c}_t$ represent the observations of $\rnd{w}_t$, which can be used to learn $\bar{w}$.

\subsection{$\dcmklucb$ Algorithm}
\label{sec:algorithm}

\begin{algorithm}[t]
  \caption{$\dcmklucb$ for solving DCM bandits.}
  \label{alg:ucb}
  \begin{algorithmic}
    \STATE // Initialization
    \STATE Observe $\rnd{w}_0 \sim P_\textsc{w}$
    \STATE $\forall e \in E: \rnd{T}_0(e) \gets 1$
    \STATE $\forall e \in E: \hat{\rnd{w}}_1(e) \gets \rnd{w}_0(e)$
    \STATE
    \FORALL{$t = 1, \dots, n$}
      \FORALL{$e = 1, \dots, L$}
        \STATE Compute UCB $\rnd{U}_t(e)$ using \eqref{eq:UCB}
      \ENDFOR
      \STATE
      \STATE // Recommend and observe
      \STATE $\rnd{A}_t \gets \argmax_{A \in \Pi_K(E)} f(A, \rnd{U}_t, \tilde{v})$
      \STATE Recommend $\rnd{A}_t$ and observe clicks $\rnd{c}_t \in \set{0, 1}^K$
      \STATE $\rnd{C}^\mathrm{last}_t \gets \max \set{k \in [K]: \rnd{c}_t(k) = 1}$
      \STATE
      \STATE // Update statistics
      \STATE $\forall e \in E: \rnd{T}_t(e) \gets \rnd{T}_{t - 1}(e)$
      \FORALL{$k = 1, \dots, \min \set{\rnd{C}^\mathrm{last}_t, K}$}
        \STATE $e \gets \rnd{a}^t_k$
        \STATE $\rnd{T}_t(e) \gets \rnd{T}_t(e) + 1$
        \STATE $\displaystyle \hat{\rnd{w}}_{\rnd{T}_t(e)}(e) \gets
        \frac{\rnd{T}_{t - 1}(e) \hat{\rnd{w}}_{\rnd{T}_{t - 1}(e)}(e) + \rnd{c}_t(k)}{\rnd{T}_t(e)}$
      \ENDFOR
    \ENDFOR
  \end{algorithmic}
\end{algorithm}

Our proposed algorithm, $\dcmklucb$, is described in \cref{alg:ucb}. It belongs to the family of UCB algorithms and is motivated by $\klucb$ \cite{garivier11klucb}. At time $t$, $\dcmklucb$ operates in three stages. First, it computes the \emph{upper confidence bounds (UCBs)} on the attraction probabilities of all items in $E$, $\rnd{U}_t \in [0, 1]^E$. The UCB of item $e$ at time $t$ is
\begin{align}
  \rnd{U}_t(e)
  & = \max \{q \in [w, 1]: w = \hat{\rnd{w}}_{\rnd{T}_{t - 1}(e)}(e)\,, \label{eq:UCB} \\
  & \phantom{{} = {}} \rnd{T}_{t - 1}(e) \kl{w}{q} \leq \log t + 3 \log \log t\}\,, \nonumber
\end{align}
where $\kl{p}{q}$ is the \emph{Kullback-Leibler (KL) divergence} between Bernoulli random variables with means $p$ and $q$; $\hat{\rnd{w}}_s(e)$ is the average of $s$ observed weights of item $e$; and $\rnd{T}_t(e)$ is the number of times that item $e$ is observed in $t$ steps. Since $\kl{p}{q}$ increases in $q$ for $q \geq p$, our UCB can be computed efficiently. After this, $\dcmklucb$ selects a list of $K$ items with largest UCBs
\begin{align*}
  \textstyle
  \rnd{A}_t = \argmax_{A \in \Pi_K(E)} f(A, \rnd{U}_t, \tilde{v})
\end{align*}
and recommends it, where $\tilde{v} \in [0, 1]^K$ is any vector whose entries are ordered in the same way as in $\bar{v}$. The selection of $\rnd{A}_t$ can be implemented efficiently in $O([L + K] \log K)$ time, by placing the item with the $k$-th largest UCB to the $k$-th most terminating position. Finally, after the user provides feedback $\rnd{c}_t$, $\dcmklucb$ updates its estimate of $\bar{w}(e)$ for any item $e$ up to position $\min \set{\rnd{C}^\mathrm{last}_t, K}$, as discussed in \cref{sec:learning without accessing rewards}.

$\dcmklucb$ is initialized with one sample of the attraction weight per item. Such a sample can be obtained in at most $L$ steps as follows \cite{kveton15cascading}. At time $t \in [L]$, item $t$ is placed at the first position. Since the first position in the DCM is always examined, $\rnd{c}_t(1)$ is guaranteed to be a sample of the attraction weight of item $t$.

%!TEX root = Paper.tex

\section{Analysis}
\label{sec:analysis}

This section is devoted to the analysis of DCM bandits. In \cref{sec:simple upper bound}, we analyze the regret of $\dcmklucb$ under the assumption that all termination probabilities are identical. This simpler case illustrates the key ideas in our proofs. In \cref{sec:upper bound}, we derive a general upper bound on the regret of $\dcmklucb$. In \cref{sec:lower bound}, we derive a lower bound on the regret in DCM bandits when all termination probabilities are identical. All supplementary lemmas are proved in \cref{sec:appendix}.

For simplicity of exposition and without loss of generality, we assume that the attraction probabilities of items satisfy $\bar{w}(1) \geq \ldots \geq \bar{w}(L)$ and that the termination probabilities of positions satisfy $\bar{v}(1) \geq \ldots \geq \bar{v}(K)$. In this setting, the \emph{optimal solution} is $A^\ast = (1, \dots, K)$. We say that item $e$ is \emph{optimal} when $e \in [K]$ and that item $e$ is \emph{suboptimal} when $e \in E \setminus [K]$. The \emph{gap} between the attraction probabilities of suboptimal item $e$ and optimal item $e^\ast$,
\begin{align*}
  \Delta_{e, e^\ast} = \bar{w}(e^\ast) - \bar{w}(e)\,,
\end{align*}
characterizes the hardness of discriminating the items. We also define the \emph{maximum attraction probability} as $p_{\max} = \bar{w}(1)$ and $\alpha = (1 - p_{\max})^{K - 1}$. In practice, $p_{\max}$ tends to be small and therefore $\alpha$ is expected to be large, unless $K$ is also large.

The key idea in our analysis is the reduction to cascading bandits \cite{kveton15cascading}. We define the \emph{cascade reward} for $i \in [K]$ recommended items as
\begin{align*}
  f_i(A, w) & = 1 - \prod_{k = 1}^{i} (1 - w(a_k))
\end{align*}
and the corresponding \emph{expected cumulative cascade regret} $R_i(n) = \EE{\sum_{t = 1}^n (f_i(A^\ast, \rnd{w}_t) - f_i(\rnd{A}_t, \rnd{w}_t))}$. We bound the cascade regret of $\dcmklucb$ below.

\begin{proposition}
\label{thm:cascade regret} For any $i \in [K]$ and $\eps > 0$, the expected $n$-step cascade regret of $\dcmklucb$ is bounded as
\begin{align*}
  R_i(n)
  & \leq \sum_{e = i + 1}^L
  \frac{(1 + \eps) \Delta_{e, i} (1 + \log(1 / \Delta_{e, i}))}{\kl{\bar{w}(e)}{\bar{w}(i)}} \times {} \\
  & \qquad\qquad\ \!\! (\log n + 3 \log \log n) + C\,,
\end{align*}
where $C = i L \frac{C_2(\eps)}{n^{\beta(\eps)}} + 7 i \log \log n$, and $C_2(\eps)$ and $\beta(\eps)$ are defined in \citet{garivier11klucb}.
\end{proposition}
\begin{proof}
The proof is identical to that of Theorem 3 in \citet{kveton15cascading} for the following reason. Our confidence radii have the same form as those in $\cascadeklucb$; and for any $\rnd{A}_t$ and $\rnd{w}_t$, $\dcmklucb$ is guaranteed to observe at least as many entries of $\rnd{w}_t$ as $\cascadeklucb$.
\end{proof}

To simplify the presentation of our proofs, we introduce \emph{or} function $V: [0, 1]^K \to [0, 1]$, which is defined as $V(x) = 1 - \prod_{k = 1}^K (1 - x_k)$. For any vectors $x$ and $y$ of length $K$, we write $x \geq y$ when $x_k \geq y_k$ for all $k \in [K]$. We denote the component-wise product of vectors $x$ and $y$ by $x \odot y$, and the restriction of $x$ to $A \in \Pi_K(E)$ by $x|_{A}$. The latter has precedence over the former. The expected reward can be written in our new notation as $f(A, \bar{w}, \bar{v}) = V(\bar{w}|_A \odot \bar{v})$.

\subsection{Equal Termination Probabilities}
\label{sec:simple upper bound}

Our first upper bound is derived under the assumption that all terminations probabilities are the same. The main steps in our analysis are the following two lemmas, which relate our objective to a linear function.

\begin{lemma}
\label{lem:linear upper bound} Let $x, y \in [0, 1]^K$ satisfy $x \geq y$. Then
\begin{align*}
  V(x) - V(y) \leq \sum_{k = 1}^K x_k - \sum_{k = 1}^K y_k\,.
\end{align*}
\end{lemma}

\begin{lemma}
\label{lem:linear lower bound} Let $x, y \in [0, p_{\max}]^K$ satisfy $x \geq y$. Then
\begin{align*}
  \alpha \left[\sum_{k = 1}^K x_k - \sum_{k = 1}^K y_k\right] \leq V(x) - V(y)\,,
\end{align*}
where $\alpha = (1 - p_{\max})^{K - 1}$.
\end{lemma}

Now we present the main result of this section.

\begin{theorem}
\label{thm:simple upper bound} Let $\bar{v}(k) = \gamma$ for all $k \in [K]$ and $\eps > 0$. Then the expected $n$-step regret of $\dcmklucb$ is bounded as
\begin{align*}
  R(n)
  & \leq \frac{\gamma}{\alpha} \sum_{e = K + 1}^L
  \frac{(1 + \eps) \Delta_{e, K} (1 + \log(1 / \Delta_{e, K}))}{\kl{\bar{w}(e)}{\bar{w}(K)}} \times {} \\
  & \hspace{0.73in} (\log n + 3 \log \log n) + C\,,
\end{align*}
where $C = \frac{\gamma}{\alpha} \left(K L \frac{C_2(\eps)}{n^{\beta(\eps)}} + 7 K \log \log n\right)$, and $C_2(\eps)$ and $\beta(\eps)$ are from \cref{thm:cascade regret}.
\end{theorem}
\begin{proof}
Let $\rnd{R}_t = R(\rnd{A}_t, \rnd{w}_t, \rnd{v}_t)$ be the stochastic regret at time $t$ and
\begin{align*}
  \cH_t = (\rnd{A}_1, \rnd{c}_1, \dots, \rnd{A}_{t - 1}, \rnd{c}_{t - 1}, \rnd{A}_t)
\end{align*}
be the \emph{history} of the learning agent up to choosing list $\rnd{A}_t$, the first $t - 1$ observations and $t$ actions. By the tower rule, we have $R(n) = \sum_{t = 1}^n \EE{\condEE{\rnd{R}_t}{\cH_t}}$, where
\begin{align*}
  \condEE{\rnd{R}_t}{\cH_t}
  & = f(A^\ast, \bar{w}, \bar{v}) - f(\rnd{A}_t, \bar{w}, \bar{v}) \\
  & = V(\bar{w}|_{A^\ast} \odot \bar{v}) - V(\bar{w}|_{\rnd{A}_t} \odot \bar{v})\,.
\end{align*}
Now note that the items in $A^\ast$ can be permuted such that any optimal item in $\rnd{A}_t$ matches the corresponding item in $A^\ast$, since $\bar{v}(k) = \gamma$ for all $k \in [K]$ and $V(x)$ is invariant to the permutation of $x$. Then $\bar{w}|_{A^\ast} \odot \bar{v} \geq \bar{w}|_{\rnd{A}_t} \odot \bar{v}$ and we can bound $\condEE{\rnd{R}_t}{\cH_t}$ from above by \cref{lem:linear upper bound}. Now we apply \cref{lem:linear lower bound} and get
\begin{align*}
  \condEE{\rnd{R}_t}{\cH_t}
  & \leq \gamma \left[\sum_{k = 1}^K \bar{w}(a^\ast_k) - \sum_{k = 1}^K \bar{w}(\rnd{a}^t_k)\right] \\
  & \leq \frac{\gamma}{\alpha} \left[f_K(A^\ast, \bar{w}) - f_K(\rnd{A}_t, \bar{w})\right]\,.
\end{align*}
By the definition of $R(n)$ and from the above inequality, it follows that
\begin{align*}
  R(n) \leq
  \frac{\gamma}{\alpha} \sum_{t = 1}^n \EE{f_K(A^\ast, \bar{w}) - f_K(\rnd{A}_t, \bar{w})} =
  \frac{\gamma}{\alpha} R_K(n)\,.
\end{align*}
Finally, we bound $R_K(n)$ using \cref{thm:cascade regret}.
\end{proof}

\subsection{General Upper Bound}
\label{sec:upper bound}

Our second upper bound holds for any termination probabilities. Recall that we still assume that $\dcmklucb$ knows the order of these probabilities. To prove our upper bound, we need one more supplementary lemma.

\begin{lemma}
\label{lem:decreasing order} Let $x \in [0, 1]^K$ and $x' $ be the permutation of $x$ whose entries are in decreasing order, $x'_1 \geq \ldots \geq x'_K$. Let the entries of $c \in [0, 1]^K$ be in decreasing order. Then
\begin{align*}
  V(c \odot x') - V(c \odot x) \leq \sum_{k = 1}^K c_k x'_k - \sum_{k = 1}^K c_k x_k\,.
\end{align*}
\end{lemma}

Now we present our most general upper bound.

\begin{theorem}
\label{thm:upper bound} Let $\bar{v}(1) \geq \ldots \geq \bar{v}(K)$ and $\eps > 0$. Then the expected $n$-step regret of $\dcmklucb$ is bounded as
\begin{align*}
  & R(n) \leq
  (1 + \eps) \sum_{i = 1}^K \frac{\bar{v}(i) - \bar{v}(i + 1)}{\alpha} \times {} \\
  & \ \sum_{e = i + 1}^L \frac{\Delta_{e, i} (1 + \log(1 / \Delta_{e, i}))}{\kl{\bar{w}(e)}{\bar{w}(i)}}
  (\log n + 3 \log \log n) + C\,,
\end{align*}
where $\bar{v}(K + 1) = 0$, $C = \sum_{i = 1}^K \frac{\bar{v}(i) - \bar{v}(i + 1)}{\alpha} \Big(i L \frac{C_2(\eps)}{n^{\beta(\eps)}} +\allowbreak 7 i \log \log n\Big)$, and $C_2(\eps)$ and $\beta(\eps)$ are from \cref{thm:cascade regret}.
\end{theorem}
\begin{proof}
Let $\rnd{R}_t$ and $\cH_t$ be defined as in the proof of \cref{thm:simple upper bound}. The main challenge in this proof is that we cannot apply \cref{lem:linear upper bound} as in the proof of \cref{thm:simple upper bound}, because we cannot guarantee that $\bar{w}|_{A^\ast} \odot \bar{v} \geq \bar{w}|_{\rnd{A}_t} \odot \bar{v}$ when the termination probabilities are not identical. To overcome this problem, we rewrite $\condEE{\rnd{R}_t}{\cH_t}$ as
\begin{align*}
  \condEE{\rnd{R}_t}{\cH_t} 
  & = [V(\bar{w}|_{A^\ast} \odot \bar{v}) - V(\bar{w}|_{\rnd{A}'_t} \odot \bar{v})] + {} \\
  & \phantom{{} = {}} [V(\bar{w}|_{\rnd{A}'_t} \odot \bar{v}) - V(\bar{w}|_{\rnd{A}_t} \odot \bar{v})]\,,
\end{align*}
where $\rnd{A}'_t$ is the permutation of $\rnd{A}_t$ where all items are in the decreasing order of their attraction probabilities. From the definitions of $A^\ast$ and $\rnd{A}_t'$, $\bar{w}|_{A^\ast} \odot \bar{v} \geq \bar{w}|_{\rnd{A}'_t} \odot \bar{v}$, and we can apply \cref{lem:linear upper bound} to bound the first term above. We bound the other term by \cref{lem:decreasing order} and get
\begin{align*}
  \condEE{\rnd{R}_t}{\cH_t} 
  & \leq \sum_{k = 1}^K \bar{v}(k) (\bar{w}(a^\ast_k) - \bar{w}(\rnd{a}^t_k)) \\
  & = \sum_{i = 1}^K [\bar{v}(i) - \bar{v}(i + 1)] \sum_{k = 1}^i (\bar{w}(a^\ast_k) - \bar{w}(\rnd{a}^t_k))\,,
\end{align*}
where we define $\bar{v}(K + 1) = 0$. Now we bound each term $\sum_{k = 1}^i (\bar{w}(a^\ast_k) - \bar{w}(\rnd{a}^t_k))$ by \cref{lem:linear lower bound}, and get from the definitions of $R(n)$ and $R_i(n)$ that
\begin{align*}
  R(n) \leq
  \sum_{i = 1}^K \frac{\bar{v}(i) - \bar{v}(i + 1)}{\alpha} R_i(n)\,.
\end{align*}
Finally, we bound each $R_i(n)$ using \cref{thm:cascade regret}.
\end{proof}

Note that when $\bar{v}(k) = \gamma$ for all $k \in [K]$, the above upper bound reduces to that in \cref{thm:simple upper bound}.

%!TEX root = Paper.tex

\subsection{Lower Bound}
\label{sec:lower bound}

Our lower bound is derived on the following class of problems. The ground set are $L$ items $E = [L]$ and $K$ of these items are optimal, $A^\ast \subseteq \Pi_K(E)$. The attraction probabilities of items are defined as
\begin{align*}
  \bar{w}(e) =
  \begin{cases}
    p & e \in A^\ast \\
    p - \Delta & \text{otherwise}\,,
  \end{cases}
\end{align*}
where $p$ is a common attraction probability of the optimal items, and $\Delta$ is the gap between the attraction probabilities of the optimal and suboptimal items. The number of positions is $K$ and their termination probabilities are identical, $\bar{v}(k) = \gamma$ for all positions $k \in [K]$. We denote an instance of our problem by $B_\mathrm{LB}(L, A^\ast, p, \Delta, \gamma)$; and parameterize it by $L$, $A^\ast$, $p$, $\Delta$, and $\gamma$. The key step in the proof of our lower bound is the following lemma.

\begin{lemma}
\label{lem:regret lower bound} Let $x, y \in [0, 1]^K$ satisfy $x \geq y$. Let $\gamma \in [0, 1]$. Then $V(\gamma x) - V(\gamma y) \geq \gamma [V(x) - V(y)]$.
\end{lemma}

Our lower bound is derived for consistent algorithms as in \citet{lai85asymptotically}. We say that the algorithm is \emph{consistent} if for any DCM bandit, any suboptimal item $e$, and any $\alpha > 0$, $\EE{\rnd{T}_n(e)} = o(n^\alpha)$; where $\rnd{T}_n(e)$ is the number of times that item $e$ is observed in $n$ steps, the item is placed at position $\rnd{C}^\mathrm{last}_t$ or higher for all $t \leq n$. Our lower bound is derived below.

\begin{theorem}
\label{thm:lower bound} For any DCM bandit $B_\mathrm{LB}$, the regret of any consistent algorithm is bounded from below as
\begin{align*}
  \liminf_{n \to \infty} \frac{R(n)}{\log n} \geq
  \gamma \alpha \frac{(L - K) \Delta}{\kl{p - \Delta}{p}}\,.
\end{align*}
\end{theorem}
\begin{proof}
The key idea of the proof is to reduce our problem to a cascading bandit. By the tower rule and \cref{lem:regret lower bound}, the $n$-step regret in DCM bandits is bounded from below as
\begin{align*}
  R(n) \geq \gamma \EE{\sum_{t = 1}^n (f_K(A^\ast, \rnd{w}_t) - f_K(\rnd{A}_t, \rnd{w}_t))}\,.
\end{align*}
Moreover, by the tower rule and \cref{lem:linear lower bound}, we can bound the $n$-step regret in cascading bandits from below as
\begin{align*}
  R(n)
  & \geq \gamma \alpha
  \EE{\sum_{t = 1}^n \left(\sum_{k = 1}^K \rnd{w}_t(a^\ast_k) - \sum_{k = 1}^K \rnd{w}_t(\rnd{a}^t_k)\right)} \\
  & \geq \gamma \alpha \Delta \sum_{e = K + 1}^L \EE{\rnd{T}_n(e)}\,,
\end{align*}
where the last step follows from the facts that the expected regret for recommending any suboptimal item $e$ is $\Delta$, and that the number of times that this item is recommended in $n$ steps is bounded from below by $\rnd{T}_n(e)$. Finally, for any consistent algorithm and item $e$,
\begin{align*}
  \liminf_{n \to \infty} \frac{\EE{\rnd{T}_n(e)}}{\log n} \geq
  \frac{\Delta}{\kl{p - \Delta}{p}}\,,
\end{align*}
by the same argument as in \citet{lai85asymptotically}. Otherwise, the algorithm would not be able to distinguish some instances of $B_\mathrm{LB}$ where item $e$ is optimal, and would have $\Omega(n^\alpha)$ regret for some $\alpha > 0$ on these problems. Finally, we chain the above two inequalities and this completes our proof.
\end{proof}

%!TEX root = Paper.tex

\subsection{Discussion}
\label{sec:discussion}

We derive two gap-dependent upper bounds on the $n$-step regret of $\dcmklucb$, under the assumptions that all termination probabilities are identical (\cref{thm:simple upper bound}) and that their order is known (\cref{thm:upper bound}). Both bounds are logarithmic in $n$, linear in the number of items $L$, and decrease as the number of recommended items $K$ increases. The bound in \cref{thm:simple upper bound} grows linearly with $\gamma$, the common termination probability at all positions. Since smaller $\gamma$ result in more clicks, we show that the regret decreases with more clicks. This is in line with our expectation that it is easier to learn from more feedback.

The upper bound in \cref{thm:simple upper bound} is tight on problem $B_{\text{LB}}(L,\allowbreak A^\ast = [K], p = 1 / K, \Delta, \gamma)$ from \cref{sec:lower bound}. In this problem, $1 / \alpha \leq e$ and $1 / e \leq \alpha$ when $p = 1 / K$; and then the upper bound in \cref{thm:simple upper bound} and the lower bound in \cref{thm:lower bound} reduce to
\begin{align*}
  & \textstyle O\left(\gamma (L - K) \frac{\Delta (1 + \log(1 / \Delta))}{\kl{p - \Delta}{p}} \log n\right)\,, \\
  & \textstyle \Omega\left(\gamma (L - K) \frac{\Delta}{\kl{p - \Delta}{p}} \log n\right)\,,
\end{align*}
respectively. The bounds match up to $\log(1 / \Delta)$.

%!TEX root = Paper.tex

\section{Experiments}
\label{sec:experiments}

We conduct three experiments. In \cref{sec:regret bound experiment}, we validate that the regret of $\dcmklucb$ scales as suggested by \cref{thm:simple upper bound}. In \cref{sec:baseline experiment}, we compare $\dcmklucb$ to multiple baselines. Finally, in \cref{sec:real-world experiment}, we evaluate $\dcmklucb$ on a real-world dataset.

\subsection{Regret Bounds}
\label{sec:regret bound experiment}

\begin{figure*}[t]
  \centering
  \includegraphics[width=2.2in, bb=2.75in 4.5in 5.75in 6.5in]{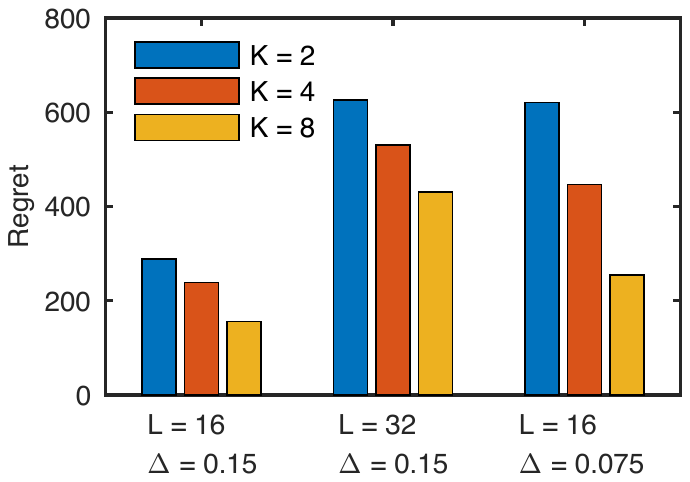}
  \includegraphics[width=2.2in, bb=2.75in 4.5in 5.75in 6.5in]{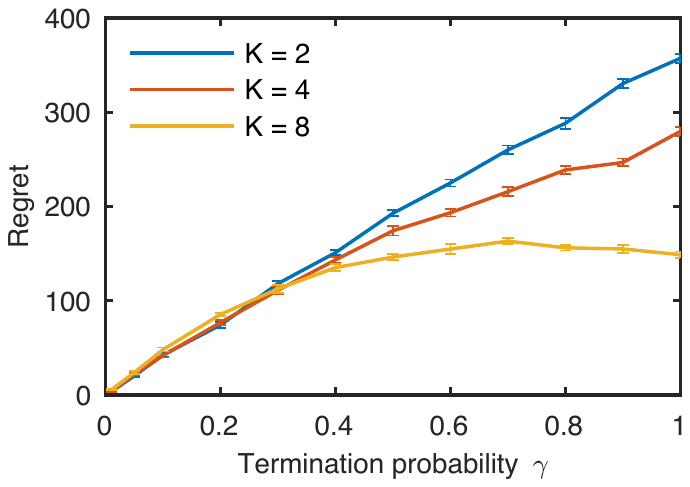}
  \includegraphics[width=2.2in, bb=2.75in 4.5in 5.75in 6.5in]{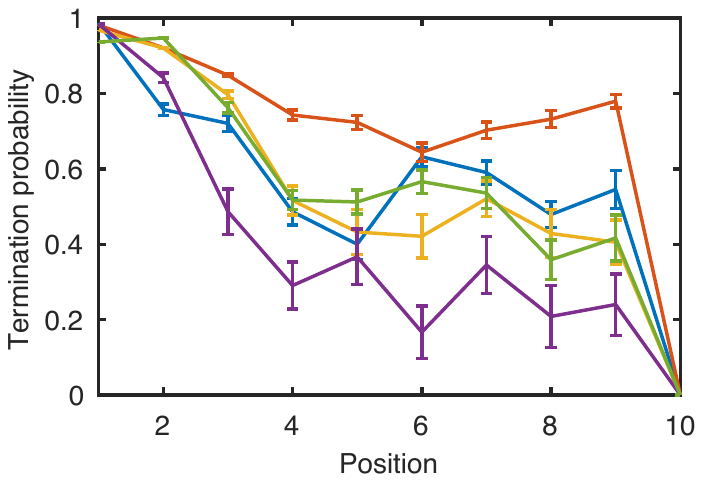} \\
  \hspace{0.07in} (a) \hspace{2in} (b) \hspace{2in} (c) \vspace{-0.1in}
  \caption{\textbf{a}. The $n$-step regret of $\dcmklucb$ in $n = 10^5$ steps on the problem in \cref{sec:regret bound experiment}. All results are averaged over $20$ runs. \textbf{b}. The $n$-step regret of $\dcmklucb$ as a function of the common termination probability $\gamma$ and $K$. \textbf{c}. The termination probabilities in the DCMs of $5$ most frequent queries in the Yandex dataset.}
\label{fig:1}
\end{figure*}

\begin{figure*}[t]
  \centering
  \includegraphics[width=2.2in, bb=2.75in 4.5in 5.75in 6.5in]{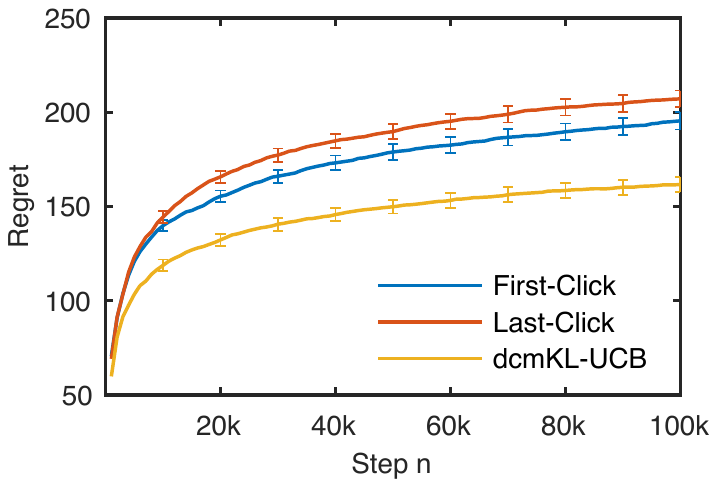}
  \includegraphics[width=2.2in, bb=2.75in 4.5in 5.75in 6.5in]{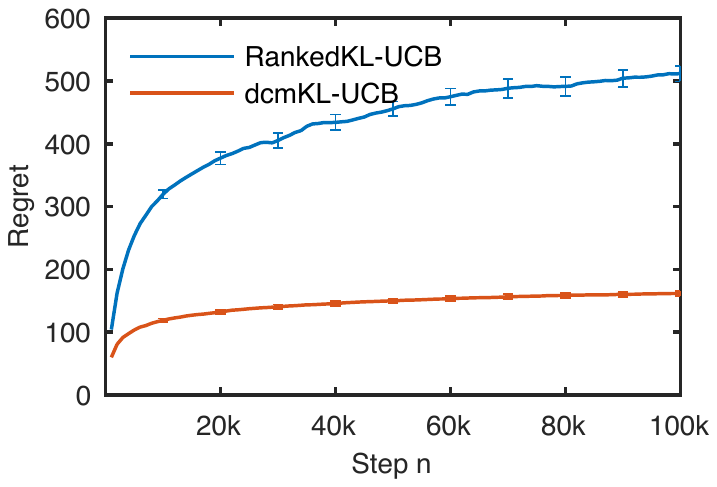}
  \includegraphics[width=2.2in, bb=2.75in 4.5in 5.75in 6.5in]{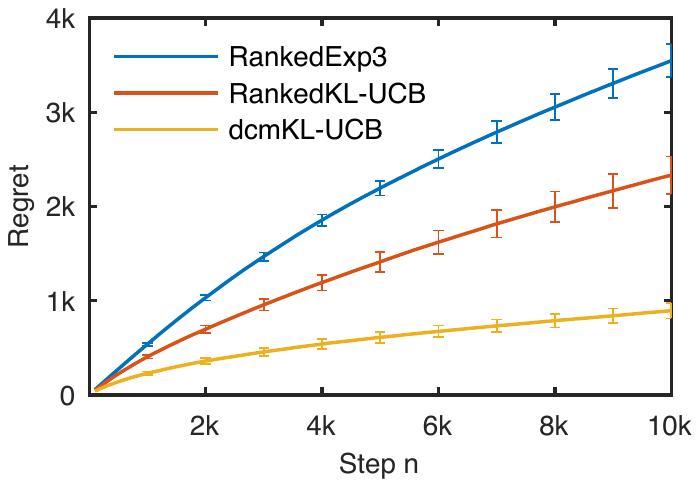} \\
  \hspace{0.07in} (a) \hspace{2in} (b) \hspace{2in} (c) \vspace{-0.1in}
  \caption{\textbf{a}. The $n$-step regret of $\dcmklucb$ and two heuristics on the problem in \cref{sec:baseline experiment}. \textbf{b}. The $n$-step regret of $\dcmklucb$ and $\rankedklucb$ on the same problem. \textbf{c}. The $n$-step regret of $\dcmklucb$, $\rankedklucb$, and $\rankedexpthree$ in the Yandex dataset.}
  \label{fig:2}
\end{figure*}

In the first experiment, we validate the behavior of our upper bound in \cref{thm:simple upper bound}. We experiment with the class of problems $B_{\text{LB}}(L, A^\ast = [K], p = 0.2, \Delta, \gamma)$, which is presented in \cref{sec:lower bound}. We vary $L$, $K$, $\Delta$, and $\gamma$; and report the regret of $\dcmklucb$ in $n = 10^5$ steps.

Figure~\ref{fig:1}a shows the $n$-step regret of $\dcmklucb$ as a function of $L$, $K$, and $\Delta$ for $\gamma = 0.8$. We observe three trends. First, the regret increases when the number of items $L$ increases. Second, the regret decreases when the number of recommended items $K$ increases. These dependencies are suggested by our $O(L - K)$ upper bound. Finally, we observe that the regret increases when $\Delta$ decreases. 

Figure~\ref{fig:1}b shows the $n$-step regret of $\dcmklucb$ as a function of $\gamma$ and $K$, for $L = 16$ and $\Delta = 0.15$. We observe that the regret grows linearly with $\gamma$, as suggested by \cref{thm:simple upper bound}, when $p < 1 / K$. This trend is less prominent when $p > 1 / K$. We believe that this is because the upper bound in \cref{thm:simple upper bound} is loose when $\alpha = (1 - p)^{K - 1}$ is small, and this happens when $p$ is large.

\subsection{First Click, Last Click, and Ranked Bandits}
\label{sec:baseline experiment}

In the second experiment, we compare $\dcmklucb$ to two single-click heuristics and ranked bandits (\cref{sec:related work}). The heuristics are motivated by $\cascadeklucb$, which learns from a single click \cite{kveton15cascading}. The first heuristic is $\dcmklucb$ where the feedback $\rnd{c}_t$ is altered such that it contains only the first click. This method can be viewed as a conservative extension of $\cascadeklucb$ to multiple clicks and we call it $\firstclick$. The second heuristic is $\dcmklucb$ where the feedback $\rnd{c}_t$ is modified such that it contains only the last click. This method was suggested by \citet{kveton15cascading} and we call it $\lastclick$. We also compare $\dcmklucb$ to $\rankedklucb$, which is a ranked bandit with $\klucb$. The base algorithm in $\rankedklucb$ is the same as in $\dcmklucb$, and therefore we believe that this comparison is fair. All methods are evaluated on problem $B_{\text{LB}}(L = 16, A^\ast = [4], p = 0.2, \Delta = 0.15, \gamma = 0.5)$ from \cref{sec:lower bound}.

The regret of $\dcmklucb$, $\firstclick$, and $\lastclick$ is shown in Figure~\ref{fig:2}a. The regret of $\dcmklucb$ is clearly the lowest among all compared methods. We conclude that $\dcmklucb$ outperforms both baselines because it does not discard or misinterpret any feedback in $\rnd{c}_t$.

The regret of $\rankedklucb$ and $\dcmklucb$ is reported in Figure~\ref{fig:2}b. We observe that the regret of $\rankedklucb$ is three times higher than that of $\dcmklucb$. Note that $K =\allowbreak 4$. Therefore, this validates our hypothesis that $\dcmklucb$ can learn $K$ times faster than a ranked bandit, because the regret of $\dcmklucb$ is $O(L - K)$ (\cref{sec:discussion}) while the regret in ranked bandits is $O(K L)$ (\cref{sec:related work}).

\subsection{Real-World Experiment}
\label{sec:real-world experiment}

In the last experiment, we evaluate $\dcmklucb$ on the \emph{Yandex} dataset \cite{yandex}, a search log of $35\text{M}$ search sessions. In each query, the user is presented $10$ web pages and may click on multiple pages. We experiment with $20$ most frequent queries from our dataset and estimate one DCM per query, as in \citet{guo09efficient}. We compare $\dcmklucb$ to $\rankedklucb$ (\cref{sec:baseline experiment}) and $\rankedexpthree$. The latter is a ranked bandit with $\expthree$, which can learn correlations among recommended positions. We parameterize $\expthree$ as suggested in \citet{auer95gambling}. All compared algorithms assume that higher ranked positions are more valuable, as this would be expected in practice. This is not necessarily true in our DCMs (Figure~\ref{fig:1}c). However, this assumption is quite reasonable because most of our DCMs have the following structure. The first position is the most terminating and the most attractive item tends to be much more attractive than the other items. Therefore, any solution that puts the most attractive item at the first position performs well. All methods are evaluated by their average regret over all $20$ queries, with $5$ runs per query. 

Our results are reported in Figure~\ref{fig:2}c and we observe that $\dcmklucb$ outperforms both ranked bandits. At $n = 10\text{k}$, for instance, the regret of $\dcmklucb$ is at least two times lower than that of our best baseline. This validates our hypothesis that $\dcmklucb$ can learn much faster than ranked bandits (\cref{sec:baseline experiment}), even in practical problems where the model of the world is likely to be misspecified.

%!TEX root = Paper.tex

\section{Related Work}
\label{sec:related work}

Our work is closely related to \emph{cascading bandits} \cite{kveton15cascading,combes15learning}. Cascading bandits are an online learning variant of the cascade model of user behavior in web search \cite{craswell08experimental}. \citet{kveton15cascading} proposed a learning algorithm for these problems, $\cascadeklucb$; bounded its regret; and proved a matching lower bound up to logarithmic factors. The main limitation of cascading bandits is that they cannot learn from multiple clicks. DCM bandits are a generalization of cascading bandits that allows multiple clicks.

\emph{Ranked bandits} are a popular approach in learning to rank \cite{radlinski08learning,slivkins13ranked}. The key idea in ranked bandits is to model each position in the recommended list as a separate bandit problem, which is solved by some \emph{base bandit algorithm}. In general, the algorithms for ranked bandits learn $(1 - 1 / e)$ approximate solutions and their regret is $O(K L)$, where $L$ is the number of items and $K$ is the number of recommended items. We compare $\dcmklucb$ to ranked bandits in \cref{sec:experiments}.

DCM bandits can be viewed as a partial monitoring problem where the reward, the satisfaction of the user, is unobserved. Unfortunately, general algorithms for partial monitoring \cite{agrawal89asymptotically,bartok12adaptive,bartok12partial,bartok14partial} are not suitable for DCM bandits because their number of actions is exponential in the number of recommended items $K$. \citet{lin14combinatorial} and \citet{kveton15combinatorial} proposed algorithms for combinatorial partial monitoring. The feedback models in these algorithms are different from ours and therefore they cannot solve our problem.

The feasible set in DCM bandits is combinatorial, any list of $K$ items out of $L$ is feasible, and the learning agent observes the weights of individual items. This setting is similar to stochastic combinatorial semi-bandits, which are often studied with linear reward functions \cite{gai12combinatorial,chen13combinatorial,kveton14matroid,kveton15tight,wen15efficient,combes15combinatorial}. The differences in our work are that the reward function is non-linear and that the feedback model is less than semi-bandit, because the learning agent does not observe the attraction weights of all recommended items.

%!TEX root = Paper.tex

\section{Conclusions}
\label{sec:conclusions}

In this paper, we study a learning variant of the dependent click model, a popular click model in web search \cite{chuklin15click}. We propose a practical online learning algorithm for solving it, $\dcmklucb$, and prove gap-dependent upper bounds on its regret. The design and analysis of our algorithm are challenging because the learning agent does not observe rewards. Therefore, we propose an additional assumption that allows us to learn efficiently. Our analysis relies on a novel reduction to a single-click model, which still preserves the multi-click character of our model. We evaluate $\dcmklucb$ on several problems and observe that it performs well even when our modeling assumptions are violated.

We leave open several questions of interest. For instance, the upper bound in \cref{thm:simple upper bound} is linear in the common termination probability $\gamma$. However, Figure~\ref{fig:1}b shows that the regret of $\dcmklucb$ is not linear in $\gamma$ for $p > 1 / K$. This indicates that our upper bounds can be improved. We also believe that our approach can be contextualized, along the lines of \citet{zong16cascading}; and extended to more complex cascading models, such as influence propagation in social networks, along the lines of \citet{wen16influence}.

To the best of our knowledge, this paper presents the first practical and regret-optimal online algorithm for learning to rank with multiple clicks in a cascade-like click model. We believe that our work opens the door to further developments in other, perhaps more complex and complete, instances of learning to rank with multiple clicks.

\subsubsection*{Acknowledgments}

This work was supported by the Alberta Innovates Technology Futures and NSERC.

\clearpage

\bibliographystyle{icml2016}
\bibliography{References}

%!TEX root = Paper.tex

\clearpage
\onecolumn
\appendix

\section{Proofs}
\label{sec:appendix}

\newtheorem*{lem:linear upper bound}{\cref{lem:linear upper bound}}
\begin{lem:linear upper bound}
Let $x, y \in [0, 1]^K$ satisfy $x \geq y$. Then
\begin{align*}
  V(x) - V(y) \leq \sum_{k = 1}^K x_k - \sum_{k = 1}^K y_k\,.
\end{align*}
\end{lem:linear upper bound}
\begin{proof}
Let $x = (x_1, \dots, x_K)$ and
\begin{align*}
  d(x) =
  \sum_{k = 1}^K x_k - V(x) =
  \sum_{k = 1}^K x_k - \left[1 - \prod_{k = 1}^K (1 - x_k)\right]\,.
\end{align*}
Our claim can be proved by showing that $d(x) \geq 0$ and $\frac{\partial}{\partial x_i} d(x) \geq 0$, for any $x \in [0, 1]^K$ and $i \in [K]$. First, we show that $d(x) \geq 0$ by induction on $K$. The claim holds trivially for $K = 1$. For any $K \geq 2$,
\begin{align*}
  d(x) =
  \sum_{k = 1}^{K - 1} x_k - \left[1 - \prod_{k = 1}^{K - 1} (1 - x_k)\right] +
  \underbrace{x_K - x_K \prod_{k = 1}^{K - 1} (1 - x_k)}_{\geq 0} \geq
  0\,,
\end{align*}
where $\displaystyle \sum_{k = 1}^{K - 1} x_k - \left[1 - \prod_{k = 1}^{K - 1} (1 - x_k)\right] \geq 0$ holds by our induction hypothesis. Second, we note that
\begin{align*}
  \frac{\partial}{\partial x_i} d(x) = 1 - \prod_{k \neq i} (1 - x_k) \geq 0\,.
\end{align*}
This concludes our proof.
\end{proof}

\newtheorem*{lem:linear lower bound}{\cref{lem:linear lower bound}}
\begin{lem:linear lower bound}
Let $x, y \in [0, p_{\max}]^K$ satisfy $x \geq y$. Then
\begin{align*}
  \alpha \left[\sum_{k = 1}^K x_k - \sum_{k = 1}^K y_k\right] \leq V(x) - V(y)\,,
\end{align*}
where $\alpha = (1 - p_{\max})^{K - 1}$.
\end{lem:linear lower bound}
\begin{proof}
Let $x = (x_1, \dots, x_K)$ and
\begin{align*}
  d(x) =
  V(x) - \alpha \sum_{k = 1}^K x_k =
  1 - \prod_{k = 1}^K (1 - x_k) - (1 - p_{\max})^{K - 1} \sum_{k = 1}^K x_k\,.
\end{align*}
Our claim can be proved by showing that $d(x) \geq 0$ and $\frac{\partial}{\partial x_i} d(x) \geq 0$, for any $x \in [0, p_{\max}]^K$ and $i \in [K]$. First, we show that $d(x) \geq 0$ by induction on $K$. The claim holds trivially for $K = 1$. For any $K \geq 2$,
\begin{align*}
  d(x) =
  1 - \prod_{k = 1}^{K - 1} (1 - x_k) - (1 - p_{\max})^{K - 1} \sum_{k = 1}^{K - 1} x_k +
  \underbrace{x_K \prod_{k = 1}^{K - 1} (1 - x_k) - x_K (1 - p_{\max})^{K - 1}}_{\geq 0} \geq
  0\,,
\end{align*}
where $\displaystyle 1 - \prod_{k = 1}^{K - 1} (1 - x_k) - (1 - p_{\max})^{K - 1} \sum_{k = 1}^{K - 1} x_k \geq 0$ holds because $\displaystyle 1 - \prod_{k = 1}^{K - 1} (1 - x_k) - (1 - p_{\max})^{K - 2} \sum_{k = 1}^{K - 1} x_k \geq 0$, which holds by our induction hypothesis; and the remainder is non-negative because $1 - x_k \geq 1 - p_{\max}$ for any $k \in [K]$. Second, note that
\begin{align*}
  \frac{\partial}{\partial x_i} d(x) = \prod_{k \neq i} (1 - x_k) - (1 - p_{\max})^{K - 1} \geq 0\,.
\end{align*}
This concludes our proof.
\end{proof}

\newtheorem*{lem:decreasing order}{\cref{lem:decreasing order}}
\begin{lem:decreasing order}
Let $x \in [0, 1]^K$ and $x' $ be the permutation of $x$ whose entries are in decreasing order, $x'_1 \geq \ldots \geq x'_K$. Let the entries of $c \in [0, 1]^K$ be in decreasing order. Then
\begin{align*}
  V(c \odot x') - V(c \odot x) \leq \sum_{k = 1}^K c_k x'_k - \sum_{k = 1}^K c_k x_k\,.
\end{align*}
\end{lem:decreasing order}
\begin{proof}
Note that our claim is equivalent to proving
\begin{align*}
  1 - \prod_{k = 1}^K (1 - c_k x'_k) - \left[1 - \prod_{k = 1}^K (1 - c_k x_k)\right] \leq
  \sum_{k = 1}^K c_k x'_k - \sum_{k = 1}^K c_k x_k\,.
\end{align*}
If $x = x'$, our claim holds trivially. If $x \neq x'$, there must exist indices $i$ and $j$ such that $i < j$ and $x_i < x_j$. Let $\tilde{x}$ be the same vector as $x$ where entries $x_i$ and $x_j$ are exchanged, $\tilde{x}_i = x_j$ and $\tilde{x}_j = x_i$. Since $i < j$, $c_i \geq c_j$. Let
\begin{align*}
  X_{\mhyphen i, \mhyphen j} = \prod_{k \neq i, j}(1 - c_k x_k)\,.
\end{align*}
Then
\begin{align*}
  1 - \prod_{k = 1}^K (1 - c_k x'_k) - \left[1 - \prod_{k = 1}^K (1 - c_k x_k)\right]
  & = X_{\mhyphen i, \mhyphen j} \left((1 - c_i x_i) (1 - c_j x_j) - (1 - c_i \tilde{x}_i) (1 - c_j \tilde{x}_j)\right) \\
  & = X_{\mhyphen i, \mhyphen j} \left((1 - c_i x_i) (1 - c_j x_j) - (1 - c_i x_j) (1 - c_j x_i)\right) \\
  & = X_{\mhyphen i, \mhyphen j} \left(- c_i x_i - c_j x_j + c_i x_j + c_j x_i\right) \\
  & = X_{\mhyphen i, \mhyphen j} (c_i - c_j) (x_j - x_i) \\
  & \leq (c_i - c_j) (x_j - x_i) \\
  & = c_i x_j + c_j x_i - c_i x_i - c_j x_j \\
  & = c_i \tilde{x}_i + c_j \tilde{x}_j - c_i x_i - c_j x_j \\
  & = \sum_{k = 1}^K c_k \tilde{x}_k - \sum_{k = 1}^K c_k x_k\,,
\end{align*}
where the inequality is by our assumption that $(c_i - c_j) (x_j - x_i) \geq 0$. If $\tilde{x} = x'$, we are finished. Otherwise, we repeat the above argument until $x = x'$.
\end{proof}

\newtheorem*{lem:regret lower bound}{\cref{lem:regret lower bound}}
\begin{lem:regret lower bound}
Let $x, y \in [0, 1]^K$ satisfy $x \geq y$. Let $\gamma \in [0, 1]$. Then
\begin{align*}
  V(\gamma x) - V(\gamma y) \geq
  \gamma [V(x) - V(y)]\,.
\end{align*}
\end{lem:regret lower bound}
\begin{proof}
Note that our claim is equivalent to proving
\begin{align*}
  \prod_{k = 1}^K (1 - \gamma y_k) - \prod_{k = 1}^K (1 - \gamma x_k) \geq
  \gamma \left[\prod_{k = 1}^K (1 - y_k) - \prod_{k = 1}^K (1 - x_k)\right]\,.
\end{align*}
The proof is by induction on $K$. To simplify exposition, we define the following shorthands
\begin{align*}
  X_i = \prod_{k = 1}^i (1 - x_k)\,, \quad
  X^\gamma_i = \prod_{k = 1}^i (1 - \gamma x_k)\,, \quad
  Y_i = \prod_{k = 1}^i (1 - y_k)\,, \quad
  Y^\gamma_i = \prod_{k = 1}^i (1 - \gamma y_k)\,.
\end{align*}
Our claim holds trivially for $K = 1$ because
\begin{align*}
  (1 - \gamma y_1) - (1 - \gamma x_1) = \gamma [(1 - y_1) - (1 - x_1)]\,.
\end{align*}
To prove that the claim holds for any $K$, we first rewrite $Y^\gamma_K - X^\gamma_K$ in terms of $Y^\gamma_{K - 1} - X^\gamma_{K - 1}$ as
\begin{align*}
  Y^\gamma_K - X^\gamma_K
  & = (1 - \gamma y_K) Y^\gamma_{K - 1} - (1 - \gamma x_K) X^\gamma_{K - 1} \\
  & = Y^\gamma_{K - 1} - \gamma y_K Y^\gamma_{K - 1} - X^\gamma_{K - 1} +
  \gamma y_K X^\gamma_{K - 1} + \gamma (x_K - y_K) X^\gamma_{K - 1} \\
  & = (1 - \gamma y_K) (Y^\gamma_{K - 1} - X^\gamma_{K - 1}) + \gamma (x_K - y_K) X^\gamma_{K - 1}\,.
\end{align*}
By our induction hypothesis, $Y^\gamma_{K - 1} - X^\gamma_{K - 1} \geq \gamma (Y_{K - 1} - X_{K - 1})$. Moreover, $X^\gamma_{K - 1} \geq X_{K - 1}$ and $1 - \gamma y_K \geq 1 - y_K$. We apply these lower bounds to the right-hand side of the above equality and then rearrange it as
\begin{align*}
  Y^\gamma_K - X^\gamma_K
  & \geq \gamma (1 - y_K) (Y_{K - 1} - X_{K - 1}) + \gamma (x_K - y_K) X_{K - 1} \\
  & = \gamma [(1 - y_K) Y_{K - 1} - (1 - y_K + y_K - x_K) X_{K - 1}] \\
  & = \gamma [Y_K - X_K]\,.
\end{align*}
This concludes our proof.
\end{proof}

\end{document}